\documentclass[a4paper,10pt,notitlepage]{article}
\usepackage[latin1]{inputenc}
\usepackage[american]{babel}
\usepackage{tikz}
\usetikzlibrary{positioning}
\usepackage{multicol,parskip}
\newtheorem{definition}{Definition}
\newtheorem{remark}{Remark}
\newtheorem{theorem}{Theorem}[section]

\newtheorem{lemma}[theorem]{Lemma}
\newenvironment{proof}{\paragraph{Proof:}}{\hfill$\square$}
\usepackage{fancyhdr, amssymb, amsfonts, amsmath, mathtools, hyperref, lineno, blkarray}

\DeclarePairedDelimiter\abs{\lvert}{\rvert}%
\title{Higher order co-occurrence tensors for hypergraphs via face-splitting}
\author{Bryan Bischof}

\begin{document}
\maketitle
  \begin{abstract}
    A popular trick for computing a pairwise \emph{co-occurrence matrix} is the product of an incidence matrix and its transpose. We present an analog for higher order tuple co-occurrences using the \emph{face-splitting product}, or alternately known as the \emph{transpose Khatri-Rao product}. These higher order co-occurrences encode the commonality of tokens in the company of other tokens, and thus generalize the mutual information commonly studied. We demonstrate this tensor's use via a popular NLP model, and hypergraph models of similarity. 
  \end{abstract}

Studying an implicit meaning of a collection of things via their relationships with other things of the same type is a popular technique. When relationships are direct and between only two objects at a time, it's natural to use graphs. In that setting, the co-occurrences, or counts of connectedness, are a natural statistic to represent these relationships. However in many settings we are interested in relationships between more than two things; when it is useful to distinguish these higher order relationships from simply sets of pair-wise relationships, the co-occurrence statistic similarly can be extended.

This note is intended to reproduce a popular trick ($C_{\mathcal{I}}=\mathcal{I}^T\cdot \mathcal{I}$) used for co-occurrence matrices, but for the higher order co-occurrences tensor. While this trick is more of a curiosity than anything, the organization of higher order co-occurrences as a tensor provides inspiration for some related settings. 

For the sake of making this note of interest, we have included some additional bait: first, we write some simple examples for these tensor operations to provide a resource with less annoying notation, second, we suggest a fun new word embedding that extends some existing ideas in the field to take advantage of unused data that is easily available. In particular, the second construction allows for a simple model of context-based item-item embeddings. We encourage the reader to use Appendix \ref{A1} while reading the definitions, as it's much simpler.

\section*{Acknowledgements}

Eric Bunch and Janu Verma graciously talked through many of these ideas with me during preparation. Cassandra Jacobs, Dan Marthaler, Leland McInnes, Sven Schmit, Reza Sohrabi, and Hector Yee provided helpful comments and suggestions. Many others endured my explanations of these constructions.

\section{Tensor operations and constructions}

\begin{definition}
    An \emph{incidence structure} is a collection of sets, $\left\lbrace s_i \right\rbrace_{i\in I}$, containing nodes, $\left\lbrace x_j \right\rbrace_{j\in J}$, where $I$ and $J$ are indexing sets. 

    The associated \emph{incidence matrix}, $\mathcal{I}$, is the binary matrix with rows indexed by sets, and columns indexed by nodes, such that elements
    \[ 
        e_{s_i,x_j}=\begin{cases} 
          1 & x_j \in s_i \\
          0 & \text{otherwise} 
       \end{cases}
    \]

    The \emph{co-occurrences of} $x_a$ \emph{ and } $x_b$ is the order of the set $\left\lbrace s_i \mid x_a \in s_i \text{ and } x_b \in s_i \right\rbrace$.
\end{definition}

\begin{lemma}
    Incidence structures and their associated matrices are in correspondence.
\end{lemma}

The ``trick'' we've been speaking of:
\begin{lemma}
    Let $C_{\mathcal{I}}$ be the co-occurrences matrix, i.e. the matrix with rows and columns indexed by $\left\lbrace x_j \right\rbrace_{j\in J}$ and elements the co-occurrences of the indices. Then:
    \begin{equation}
        C_{\mathcal{I}}=\mathcal{I}^T\cdot \mathcal{I}
    \end{equation}  
\end{lemma}

I follow the notation of \cite{KB}:

\begin{definition}
    The \emph{Khatri-Rao} product is the ``matching columnwise'' Kronecker product, i.e. 
    \begin{equation*}
        \mathbf{C} = 
            \left[
            \begin{array} { c | c | c | c}
            \mathbf{c}_1 & \mathbf{c}_2 & \dots & \mathbf{c}_k
            \end{array}
            \right]
        ,\quad
        \mathbf{D} = 
            \left[
            \begin{array} { c | c | c | c}
            \mathbf{d}_1 & \mathbf{d}_2 & \dots &  \mathbf{d}_k
            \end{array}
            \right]
        ,
    \end{equation*}
    so that:
    \begin{equation*}
        \mathbf{C} \odot \mathbf{D}= 
            \left[
            \begin{array} { c | c | c | c}
            \mathbf{c}_1 \otimes \mathbf{d}_1 & \mathbf{c}_2 \otimes \mathbf{d}_2 & \dots & \mathbf{c}_k \otimes \mathbf{d}_k
            \end{array}
            \right]
    \end{equation*}
    Similarly, the ``matching rowwise'' Kronecker product, or the \emph{face-splitting product} is defined:
    \begin{equation*}
        \text{let } \mathbf{C} = 
            \left[
            \begin{array} { c }
            \mathbf{c}_1 \\\hline
            \mathbf{c}_2 \\\hline
            \vdots \\\hline
            \mathbf{c}_k\\
            \end{array}
            \right]
        ,\text{ and } \quad
        \mathbf{D} = 
            \left[
            \begin{array} { c }
            \mathbf{d}_1\\\hline
            \mathbf{d}_2\\\hline
            \vdots \\\hline
            \mathbf{d}_k\\
            \end{array}
            \right]
        ,
    \end{equation*}
    so that:
    \begin{equation*}
        \mathbf{C} \bullet \mathbf{D}= 
            \left[
            \begin{array} { c }
            \mathbf{c}_1 \otimes \mathbf{d}_1\\\hline 
            \mathbf{c}_2 \otimes \mathbf{d}_2\\\hline
            \vdots \\\hline
            \mathbf{c}_k \otimes \mathbf{d}_k\\
            \end{array}
            \right],
    \end{equation*} 
    also known as the transposed Khatri-Rao product.
\end{definition}

We will want to move between tensors of higher order, and matrices, by moving around the elements into different shapes. This process is a generalization of matrix \emph{vectorization}.

\begin{definition}
    The \emph{mode-p unfolding} of a tensor, $\mathcal{X} \in \mathbb{R}^{I_1 \times I_2 \times \dots \times I_N}$, is denoted, $\mathbf{X}_{(p)}$, and aligns the mode-$p$ fibers into a matrix; the tensor element $(i_1, i_2, \ldots, i_N)$ maps to matrix element $(i_p, j)$ where
    \begin{equation*}
        j = 1 + \sum^{N}_{\substack{k=1\\ k\ne p}}(i_k - 1)J_k, 
        \text{ with} \quad 
        J_k = \prod^{k-1}_{\substack{m=1\\ m\ne p}}I_m.
    \end{equation*}

    Similarly, the \emph{mode-p folding} of a matrix is denoted, $\mathbf{X}_{(-p)}$, and defined\footnote{
        A general mode-p folding is not well-defined, and requires the dimension of the output tensor to be specified. For the purposes of this note, the dimensions are always clear from context.
    } such that 
    \begin{equation*}
        \mathcal{X} = (\mathbf{X}_{(p)})_{(-p)}.
    \end{equation*}

    For a matrix $\mathbf{U}\in \mathbb{R}^{J\times I_p}$, the \emph{p-mode product of a tensor with} $\mathbf{U}$, denoted $\mathcal{X}\times_p \mathbf{U}$, is the tensor $\mathcal{Y}$ such that
    \begin{equation}\label{folding-relation}
        \mathbf{Y}_{(p)} = \mathbf{U}\cdot \mathbf{X}_{(p)}.
    \end{equation}
\end{definition}

The goal is to consider a tensor which represents the higher order co-occurrences of the incidence structure represented by $\mathcal{I}$; by this we mean that for order $k$, the \emph{co-occurrences of the multiset} $\mathbf{x}=\left\lbrace x_1, x_2, \ldots, x_k \right\rbrace$ is the order of the set $\left\lbrace s_i \mid \left\lbrace x_1, x_2, \ldots, x_k \right\rbrace \subseteq s_i \right\rbrace$, which we write $c_{\mathbf{x}}$. We denote the $k$-tensor indexed in each dimension by $X$ with elements given by these counts $\mathcal{C}^k_{\mathcal{I}}$.

We utilize multisets for index purposes, but we consider inclusion as sets, e.g. 
\begin{equation*}
    c_{aab} = c_{abb} = c_{ab}.
\end{equation*}

This also means that a co-occurrence tensor of order $k$ contains elements from all lower order tensors as subtensors.

\begin{theorem}
    Write $\mathcal{I}^{\bullet {k-1}} := \underbrace{\mathcal{I}\bullet \mathcal{I} \bullet \ldots \bullet \mathcal{I}}_{{k-1}}$, then $\mathcal{C}^{k}_{\mathcal{I}} = \mathcal{I}^{\bullet {k-1}}_{(-1)} \times_1 \mathcal{I}^T$.
\end{theorem}

\begin{proof}
    Assume $\abs{I}=n$, and $\abs{J}=m$. Consider 
    \begin{equation*}
        \mathcal{I} = 
            \left[
            \begin{array} {c c}
            \mathbf{s}_1 \\\hline
            \mathbf{s}_2 \\\hline
            \vdots \\\hline
            \mathbf{s}_m\\
            \end{array}
            \right],
    \end{equation*}
    then, iterating the face-splitting product to compute row-wise tensor products,
    \begin{equation*}
        \mathcal{I}^{\bullet {k-1}} = 
            \left[
                \underbrace{
                    \begin{array} { c }
                    \mathbf{s}_1 \otimes \mathbf{s}_1 \otimes \ldots \otimes \mathbf{s}_1\\\hline 
                    \mathbf{s}_2 \otimes \mathbf{s}_2 \otimes \ldots \otimes \mathbf{s}_2\\\hline
                    \vdots \\\hline
                    \mathbf{s}_m \otimes \mathbf{s}_m \otimes \ldots \otimes \mathbf{s}_m\\
                    \end{array}
                }_{n\cdot (k-1)}
            \right].
    \end{equation*} 
    and thus for $1\leq v \leq n\cdot(k-1)$ the multi-index in each row, we see that
    \begin{equation*}
        \mathcal{I}^T \cdot \left(\mathcal{I}^{\bullet {k-1}}\right)_v = 
            \left[
                \begin{array} { c }
                \sum^m_{j=1} e_{s_j, x_1} \cdot (\prod_v e_{s_j,v}) \\\hline
                \sum^m_{j=1} e_{s_j, x_2} \cdot (\prod_v e_{s_j,v}) \\\hline 
                \vdots \\\hline
                \sum^m_{j=1} e_{s_j, x_n} \cdot (\prod_v e_{s_j,v}) \\ 
                \end{array}
            \right]={\mathcal{C}^{k}_{\mathcal{I}}}_{(i)},
    \end{equation*} 
    the $i$'th mode unfolding. By \ref{folding-relation}, the result follows. 
\end{proof}

\section{Hypergraphs and Embeddings}

\begin{definition}
    A \emph{hypergraph} is a pair $(S,X)$ consisting of \emph{edges}, $S$, and \emph{nodes}, $X$, such that $S$ is a collection of sets, $\left\lbrace s_i \right\rbrace_{i\in I}$, containing nodes, $X$, $\left\lbrace x_j \right\rbrace_{j\in J}$ and if $x_j\in S_i$ we say $x_j$ is incident on $S_i$.

    A \emph{weighted hypergraph} is one such that each edge $s_i$ has an associated weight $w_i\in\mathbb{R}$.
\end{definition}

Notice that a hypergraph is another name for an incidence structure, and is simply a generalization of a graph such that edges in a graph are  permitted to contain more than two nodes. We now use the higher order co-occurances to build embeddings of the nodes which reflects properties of the incidence structure via distances in the embeddings.

A few examples for application:
\begin{remark}
    Given a collection of products, and historical sales, each transaction and with its collection of objects is an edge in a hypergraph of products as nodes. The co-occurance tensor describes ``bought together'' frequencies.
\end{remark}

\begin{remark} 
    Provided words $\Omega = \left\lbrace\omega_t\right\rbrace_{t\in T}$ from a corpus in sequence, and $\kappa(\omega_t)=\left\lbrace \omega_{t-r},\ldots,\omega_t,\ldots,\omega_{t+r}\right\rbrace$ contexts associated to $w_i$, then for $K$ the set of all such $\kappa(\omega_t)$, $(K, \Omega)$ is a hypergraph.
\end{remark}

\begin{remark}
    Skip-gram negative sampling (SGNS)\cite{Mikolov} framed as a shallow neural network, is used to construct the popular word-embedding model \emph{word2vec}, from the pointwise mutual information (PMI) matrix-a modification of the context-incident co-occurrences matrix. These PMI values are defined as $p_{x_i, x_j}=\log\left(\dfrac{c_{x_i, x_j}\cdot \abs{X}}{c_{x_i, x_i}\cdot c_{x_j,x_j}}\right)$. \emph{word2vec} is a representation of the words in some corpus, of small dimension relative to corpus size, such that vector operations frequently correspond to syntactical operations.
\end{remark}

If we wish to utilize our new co-occurrences tensor, we need higher order PMI values also.
\begin{definition}
    \emph{Multivariant pointwise mutual information} for a tuple $\mathbf{x}=\left\lbrace x_1, x_2, \ldots, x_k \right\rbrace$ is the ratio $p_{\mathbf{x}}=\log\left(\dfrac{c_{\mathbf{x}}\cdot \abs{X}}{\prod_{1\leq i \leq k}c_{x_k}}\right)$ referred to as \emph{Specific Correlation} \cite{VdC}
\end{definition}

\begin{remark}
    Similarly, one can consider a loss function built from ratios of co-occurrences as in (GloVe) \cite{Glove}. It's interesting to extend those notions to also include these multivariate versions.
\end{remark}

\begin{remark}
    A popular question about \emph{word2vec} models is if the optima obtained via the stochastic gradient descent and the shallow NN is in some way that which one obtains via some implicit matrix factorization of these co-occurrence matrices and/or the PMI matrices. \cite{OY} introduced the notion and later \cite{Li} described as an explicit matrix-factorization of the PMI matrix.
\end{remark}

\begin{definition}
    For $\Omega$ a corpus and $\mathcal{C}^{3}_{\Omega}$ the associated co-occurrence tensor, consider ${\mathcal{C}^{3}_{\Omega}}_{x_i:}$: the $x_i$ indexed array of $2$ tensors. For a given matrix-distance(for example Frobenius norm), $\mathcal{L}$, we construct 
    \begin{equation*}
        \min\limits_{\mathbf{Y_{i}},\mathbf{Z_{i}}}\mathcal{L}\left(\mathcal{C}^{3}_{\Omega}, \mathbf{Y_{i}}\cdot\mathbf{Z_{i}}\right)
    \end{equation*}  
    for fixed shapes, i.e. a factorization of each matrix ${\mathcal{C}^{3}_{\Omega}}_{x_i:}$, then $\left\lbrace Y_{i}, Z_{i}\right\rbrace$ is a fiber space of embeddings for ${\mathcal{C}^{3}_{\Omega}}$.
\end{definition}

This definition is motivated by the desire to include higher order co-incidence in the spaces we frequently use to simplify relationships between nodes via low-dimensional embeddings. It's interesting to observe that while $\mathcal{C}^{p}_{\Omega} \subsetneq \mathcal{C}^{q}_{\Omega}$ for all $p<q$ as sets of elements, no face properly contains the lower order faces(see Appendix \ref{B1} for a visualization). In this way, we fail to get a proper fiber embedding.\footnote{
    Note that embeddings constructed via neural networks do not parallelize well, especially for large corpi. Both computation of higher co-occurance and matrix factorization are highly parallelizable. 
}

We can extend this notion further, and define a sequence of these embeddings, $\left\lbrace{\mathcal{C}^{r}_{\Omega}}\right\rbrace_{2\leq r\leq n-1}$ with a fixed embedding dimension $d$. From each of these fiber spaces we can construct a map $\mathcal{O}^{r}:X^r\rightarrow \mathbb{R}^d$ in the following way: consider the embedding with index $x\in X^r$, and return the vector value of $x$ in that image. This sequence of maps return vectors of the same dimension, and thus, their sum lies in $\mathbb{R}^d$.

Due to our constructions above, we can easily integrate the notion of context-based embeddings:

\begin{remark}
    If instead one wishes to compute metadata-contextual similarity via a context mapping, $\mathcal{R}$, from hyperedges to a set of contexts $\left\lbrace y_k \right\rbrace_{r\in R}$, one can instead compute $\mathcal{I}^{\bullet {k-1}}_{(-1)} \times_1 \mathcal{R}^T$. Tensor factorizations, or n-mode fiber factorizations\cite{Moody} can then be used for analogous similarities with metadata enrichment.
\end{remark}

Analogous questions for matrix factorizations of the ratios of pointwise mutual information, and contextual embeddings are the subject of future work by the author.\footnote{
    Late during the preparation of this note, I became aware of the exciting work of \cite{Disco}. I hope to investigate their work and update this note with relevant connections in the future.
}

\pagebreak

\appendix

\section{Examples}\label{A1}

Unfortunately, the notation for many tensor operations are arcane. Here we present a few computation to aid in reading the previous. 

The \emph{order of a tensor} is the number of dimensions the ``cube of numbers'' representing the tensor requires, i.e. the depth of the brackets that would be printed from numpy\footnote{this formulation is due to Vicki Boykis}.

Let's begin with a corpus of words consisting of three sentences:
\begin{center}
    \begin{linenumbers}
        I like math \\
        You like math \\
        I like you.
    \end{linenumbers}
\end{center}

This provides us with an incidence matrix:

\[
    \mathcal{I}=\begin{blockarray}{ccccc}
        x_1=\text{I} & x_2=\text{like} & x_3=\text{math} & x_4=\text{you} \\
        \begin{block}{[cccc]c}
            1 & 1 & 1 & 0 & s_1 \\
            0 & 1 & 1 & 1 & s_2 \\
            1 & 1 & 0 & 1 & s_3 \\
        \end{block}
    \end{blockarray}
\]

which has the co-occurrence matrix:

\[
    \mathcal{C}_{\mathcal{I}}=\begin{blockarray}{ccccc}
        x_1=\text{I} & x_2=\text{like} & x_3=\text{math} & x_4=\text{you} \\
        \begin{block}{[cccc]c}
            2 & 2 & 1 & 1 & x_1 \\
            2 & 3 & 2 & 2 & x_2 \\
            1 & 2 & 2 & 1 & x_3 \\
            1 & 2 & 1 & 2 & x_4 \\
        \end{block}
    \end{blockarray}
\]

We compute the face-splitting square of $\mathcal{I}$:

\[
    \mathcal{I}\bullet \mathcal{I}=\begin{blockarray}{cccccccc}
        x_1\otimes x_1 & x_1\otimes x_2 & x_1\otimes x_3 & x_1\otimes x_4 & x_2\otimes x_1 & \ldots & x_4\otimes x_4 & \\
        \begin{block}{[ccccccc]c}
            1 & 1 & 1 & 0 & 1 & \ldots & 0 & s_1 \\
            0 & 0 & 0 & 0 & 0 & \ldots & 1 & s_2 \\
            1 & 1 & 0 & 1 & 1 & \ldots & 1 & s_3 \\
        \end{block}
    \end{blockarray}
\]

in the $\mathcal{I}^T\cdot \mathcal{I}^{\bullet {2}}$, the first column computation looks like this:

\begin{equation*} 
    \left[
        \begin{array} { c }
        e_{s_1, x_1\otimes x_1} \cdot e_{s_1, x_1} + e_{s_2, x_1\otimes x_1} \cdot e_{s_2, x_1} + e_{s_3, x_1\otimes x_1} \cdot e_{s_3, x_1} \\\hline
        e_{s_1, x_1\otimes x_1} \cdot e_{s_1, x_2} + e_{s_2, x_1\otimes x_1} \cdot e_{s_2, x_2} + e_{s_3, x_1\otimes x_1} \cdot e_{s_3, x_2} \\\hline 
        e_{s_1, x_1\otimes x_1} \cdot e_{s_1, x_3} + e_{s_2, x_1\otimes x_1} \cdot e_{s_2, x_3} + e_{s_3, x_1\otimes x_1} \cdot e_{s_3, x_3} \\\hline 
        e_{s_1, x_1\otimes x_1} \cdot e_{s_1, x_4} + e_{s_2, x_1\otimes x_1} \cdot e_{s_2, x_4} + e_{s_3, x_1\otimes x_1} \cdot e_{s_3, x_4} \\
        \end{array} 
    \right] = \left[
        \begin{array} { c }
        2 \\\hline
        2 \\\hline 
        1 \\\hline 
        1 \\ 
        \end{array} 
    \right] = \left[
        \begin{array} { c }
        x_1 \otimes x_1 \otimes x_1 \\\hline
        x_1 \otimes x_1 \otimes x_2 \\\hline 
        x_1 \otimes x_1 \otimes x_3 \\\hline 
        x_1 \otimes x_1 \otimes x_4 \\ 
        \end{array} 
    \right].
\end{equation*}

This is a $16\times4$ matrix, and when we do the first mode folding, we arrive at the co-occurrence 3-tensor with frontal slices, the matrices associated to fixing the last index, ${\mathcal{C}_{\mathcal{I}}}_{::j}$:

\[
    {\mathcal{C}_{\mathcal{I}}}_{::1}=\begin{blockarray}{ccccc}
        x_1 & x_2 & x_3 & x_4 \\
        \begin{block}{[cccc]c}
            2 & 2 & 1 & 1 & x_1 \\
            2 & 2 & 1 & 1 & x_2 \\
            1 & 1 & 1 & 0 & x_3 \\
            1 & 1 & 0 & 1 & x_4 \\
        \end{block}
    \end{blockarray}
\], 
\[
    {\mathcal{C}_{\mathcal{I}}}_{::2}=\begin{blockarray}{ccccc}
        x_1 & x_2 & x_3 & x_4 \\
        \begin{block}{[cccc]c}
            2 & 2 & 1 & 1 & x_1 \\
            2 & 3 & 2 & 2 & x_2 \\
            1 & 2 & 2 & 1 & x_3 \\
            1 & 2 & 1 & 2 & x_4 \\
        \end{block}
    \end{blockarray}
\],
\[
    {\mathcal{C}_{\mathcal{I}}}_{::3}=\begin{blockarray}{ccccc}
        x_1 & x_2 & x_3 & x_4 \\
        \begin{block}{[cccc]c}
            1 & 1 & 1 & 0 & x_1 \\
            1 & 2 & 2 & 1 & x_2 \\
            1 & 2 & 2 & 1 & x_3 \\
            0 & 1 & 1 & 1 & x_4 \\
        \end{block}
    \end{blockarray}
\], 
\[
    {\mathcal{C}_{\mathcal{I}}}_{::4}=\begin{blockarray}{ccccc}
        x_1 & x_2 & x_3 & x_4 \\
        \begin{block}{[cccc]c}
            1 & 1 & 0 & 1 & x_1 \\
            1 & 2 & 1 & 2 & x_2 \\
            0 & 1 & 2 & 1 & x_3 \\
            1 & 2 & 1 & 2 & x_4 \\
        \end{block}
    \end{blockarray}
\]

\section{Highest order co-occurrences images}\label{B1}

The pairwise co-occurrences in the 3-tensor are a little challenging to visualize as they do not lie in one face. For the sake of convenience here is the 3-tensor with cells corresponding to pairwise co-occurrences, for a incidence structure with six nodes.

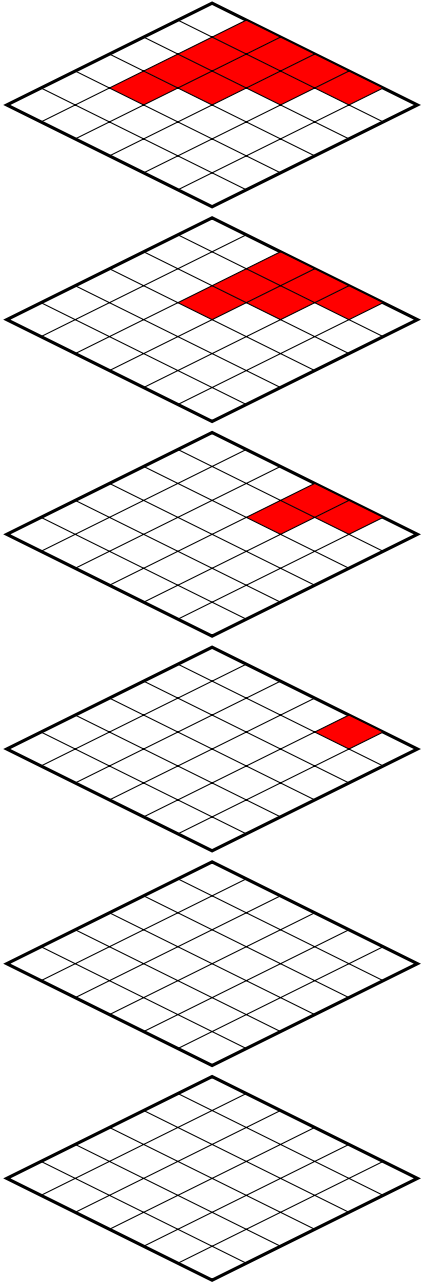
\begin{figure}\begin{tikzpicture}[scale=.9,every node/.style={minimum size=1cm},on grid]
\begin{scope}[
            yshift=0,every node/.append style={
                yslant=0.5,xslant=-1},yslant=0.5,xslant=-1
                         ]
            \fill[white,fill opacity=.75] (0,0) rectangle (3.0,3.0);
            \draw[black,very thick] (0,0) rectangle (3.0,3.0);
            \fill[red] (1.0,2.5) rectangle (1.5,2.0);
    \fill[red] (1.5,2.5) rectangle (2.0,2.0);
    \fill[red] (2.0,2.5) rectangle (2.5,2.0);
    \fill[red] (2.5,2.5) rectangle (3.0,2.0);
    \fill[red] (1.5,2.0) rectangle (2.0,1.5);
    \fill[red] (2.0,2.0) rectangle (2.5,1.5);
    \fill[red] (2.5,2.0) rectangle (3.0,1.5);
    \fill[red] (2.0,1.5) rectangle (2.5,1.0);
    \fill[red] (2.5,1.5) rectangle (3.0,1.0);
    \fill[red] (2.5,1.0) rectangle (3.0,0.5);
            \draw[step=5mm, black] (0,0) grid (3.0,3.0);
        \end{scope}
        
\begin{scope}[
            yshift=-90,every node/.append style={
                yslant=0.5,xslant=-1},yslant=0.5,xslant=-1
                         ]
            \fill[white,fill opacity=.75] (0,0) rectangle (3.0,3.0);
            \draw[black,very thick] (0,0) rectangle (3.0,3.0);
            \fill[red] (1.5,2.0) rectangle (2.0,1.5);
    \fill[red] (2.0,2.0) rectangle (2.5,1.5);
    \fill[red] (2.5,2.0) rectangle (3.0,1.5);
    \fill[red] (2.0,1.5) rectangle (2.5,1.0);
    \fill[red] (2.5,1.5) rectangle (3.0,1.0);
    \fill[red] (2.5,1.0) rectangle (3.0,0.5);
            \draw[step=5mm, black] (0,0) grid (3.0,3.0);
        \end{scope}
        
\begin{scope}[
            yshift=-180,every node/.append style={
                yslant=0.5,xslant=-1},yslant=0.5,xslant=-1
                         ]
            \fill[white,fill opacity=.75] (0,0) rectangle (3.0,3.0);
            \draw[black,very thick] (0,0) rectangle (3.0,3.0);
            \fill[red] (2.0,1.5) rectangle (2.5,1.0);
    \fill[red] (2.5,1.5) rectangle (3.0,1.0);
    \fill[red] (2.5,1.0) rectangle (3.0,0.5);
            \draw[step=5mm, black] (0,0) grid (3.0,3.0);
        \end{scope}
        
\begin{scope}[
            yshift=-270,every node/.append style={
                yslant=0.5,xslant=-1},yslant=0.5,xslant=-1
                         ]
            \fill[white,fill opacity=.75] (0,0) rectangle (3.0,3.0);
            \draw[black,very thick] (0,0) rectangle (3.0,3.0);
            \fill[red] (2.5,1.0) rectangle (3.0,0.5);
            \draw[step=5mm, black] (0,0) grid (3.0,3.0);
        \end{scope}
        
\begin{scope}[
            yshift=-360,every node/.append style={
                yslant=0.5,xslant=-1},yslant=0.5,xslant=-1
                         ]
            \fill[white,fill opacity=.75] (0,0) rectangle (3.0,3.0);
            \draw[black,very thick] (0,0) rectangle (3.0,3.0);
            
            \draw[step=5mm, black] (0,0) grid (3.0,3.0);
        \end{scope}
        
\begin{scope}[
            yshift=-450,every node/.append style={
                yslant=0.5,xslant=-1},yslant=0.5,xslant=-1
                         ]
            \fill[white,fill opacity=.75] (0,0) rectangle (3.0,3.0);
            \draw[black,very thick] (0,0) rectangle (3.0,3.0);
            
            \draw[step=5mm, black] (0,0) grid (3.0,3.0);
        \end{scope}
        
\end{tikzpicture}
    \caption{Tensor 6,3}
    \end{figure}


\begin{thebibliography}{widest entry}
    \bibitem[Bradley et. al.]{Disco} Tai-Danae Bradley, Martha Lewis, Jade Master, and Brad Theilman. Translating and Evolving: Towards a Model of Language Change in DisCoCat. \url{https://arxiv.org/abs/1811.11041}, November 2018.
    \bibitem[Kolda and Bader]{KB} T. G. Kolda, B. W. Bader. Tensor Decompositions and Applications. SIAM Review, Vol. 51, No. 3, pp. 455-500, 2009. https://doi.org/10.1137/07070111X
    \bibitem[Li et al., 2015]{Li} Yitan Li, Linli Xu, Fei Tian, Liang Jiang, Xiaowei Zhong, and Enhong Chen. Word Embedding Revisited: A New Representation Learning and Explicit Matrix Factorization Perspective. Proceedings of the Twenty-Forth International Joint Conference on Artificial Intelligence. 2015.
    \bibitem[Mikolov et al., 2013]{Mikolov} Tomas Mikolov, Ilya Sutskever, Kai Chen, Greg Corrado, and Jeff Dean. 2013. Distributed representations of words and phrases and their compositionality. In Advances in Neural Information Processing Systems 26, pages 3111-3119.
    \bibitem[Omer and Yoav, 2014]{OY} Levy Omer and Goldberg Yoav. Neural word embedding as implicit matrix factorization. Advances in Neural Information Processing Systems, 2014.
    \bibitem[Pennington et al.]{Glove} Jeffrey Pennington, Richard Socher, Christopher Manning. Glove: Global Vectors for Word Representation. Proceedings of the 2014 Conference on Empirical Methods in Natural Language Processing. 2014.
    \bibitem[Van de Cruys]{VdC} Tim van de Cruys. Two multivariate generalizations of pointwise mutual information. In C. Biemann \& E. Giesbrecht(Eds.), Proceedings of the Workshop on Distributional Semantics and Compositionality (pp. 16 - 20 ). Stroudsburg, PA : Association for Computational Linguistics . 
    \bibitem[Word Tensors]{Moody} Chris Moody. Word Tensors. \url{https://multithreaded.stitchfix.com/blog/2017/10/25/word-tensors/}
\end{thebibliography}
\end{document}